\documentclass{article} 
\usepackage{iclr2018_conference,times}
\usepackage{hyperref}
\usepackage{url}
\usepackage[utf8]{inputenc} 
\usepackage[T1]{fontenc}    
\usepackage{hyperref,amsmath,enumitem,amsthm}       
\usepackage{url}            
\usepackage{booktabs,amssymb}       
\usepackage{amsfonts}       
\usepackage{nicefrac}       
\usepackage{microtype,bm}      
\usepackage{times}
\usepackage{hyperref}
\usepackage{mathtools}
\usepackage{graphicx}
\usepackage{amsmath, mathtools}
\usepackage{amssymb}
\usepackage{hyperref}
\usepackage{mathrsfs}
\usepackage{color}
\usepackage{wrapfig}
\usepackage{caption}
\usepackage{subcaption}
\usepackage{caption}

\usepackage{booktabs}
\usepackage{multirow}
\usepackage{xr}

\newtheorem{proposition}{Proposition}

\makeatletter
\newcommand*{\rom}[1]{\expandafter\@slowromancap\romannumeral #1@}
\newcommand{\para}[1]{\vspace{-0.5ex}\textbf{#1}}
\newcommand{\xoverbrace}[2][\vphantom{\dfrac{A}{A}}]{\overbrace{#1#2}}
\newcommand{\xunderbrace}[2][\vphantom{\dfrac{A}{A}\sum_{i\neq\hat{y}}}]{\underbrace{#1#2}}

\makeatother

\usepackage[titletoc,title]{appendix}

\title{Enhancing The Reliability of \\ Out-of-distribution Image Detection in\\ Neural Networks}

\author{
Shiyu Liang\\
  Coordinated Science Lab, Department of ECE\\
  University of Illinois at Urbana-Champaign\\
  \texttt{sliang26@illinois.edu} \\
   \And
   Yixuan Li \\
  University of Wisconsin-Madison\thanks{Work done while at Cornell University.}\\
   \texttt{sharonli@cs.wisc.edu} \\
   \And
   R. Srikant \\
  Coordinated Science Lab, Department of ECE\\
  University of Illinois at Urbana-Champaign\\
  \texttt{rsrikant@illinois.edu} \\
}
\iclrfinalcopy 

\begin{document}

\maketitle

\begin{abstract}
We consider the problem of detecting {\em out-of-distribution} images in neural networks. We propose \textit{ODIN}, a simple and effective method that does not require any change to a pre-trained neural network. Our method is based on the observation that using temperature scaling and adding small perturbations to the input can  separate the softmax score distributions between in- and out-of-distribution images, allowing for more effective detection. We show in a series of experiments that \textit{ODIN} is compatible with diverse network architectures and datasets. It consistently outperforms the baseline approach~\citep{hendrycks2016baseline} by a large margin, establishing a new state-of-the-art performance on this task. For example, \textit{ODIN} reduces the false positive rate from the baseline 34.7\% to 4.3\% on the DenseNet (applied to CIFAR-10 and Tiny-ImageNet) when the true positive rate is 95\%. 

\end{abstract}

\vspace{-.3cm}
\section{Introduction}
Modern neural networks {are known to} generalize well {when the training and testing data are sampled from the same distribution~\citep{krizhevsky2012imagenet,simonyan2014very,he2016deep,cho2014learning,zhang2016understanding}.} However, when deploying neural networks {in real-world applications}, there is often very little control over the testing data distribution. {Recent works have} shown that neural networks tend to {make} high confidence predictions even for completely unrecognizable~\citep{nguyen2015deep} or irrelevant inputs~\citep{ hendrycks2016baseline,szegedy2013intriguing,moosavi2016universal}.  It has been well documented \citep{amodei2016concrete} that it is important for classifiers to be aware of uncertainty when shown new kinds of inputs, i.e., \textbf{out-of-distribution} examples. 
Therefore, being able to accurately detect out-of-distribution examples can be practically important for visual recognition tasks~\citep{krizhevsky2012imagenet,farabet2013learning,ji20133d}.


A seemingly straightforward approach of detecting out-of-distribution images is to enlarge the training set of both in- and out-of-distribution examples. However, the number of out-of-distribution examples can be infinitely many, making the re-training approach computationally expensive and intractable. Moreover, to ensure that a neural network accurately classifies in-distribution samples into correct classes while correctly detecting out-of-distribution samples, one might need to employ exceedingly large neural network architectures, which further complicates the training process. 

 
\citeauthor{hendrycks2016baseline} proposed a baseline method to detect out-of-distribution examples without further re-training networks. The method is based on an observation that a well-trained neural network tends to assign higher softmax scores to in-distribution examples than out-of-distribution examples. In this paper, we go further. We observe that after using temperature scaling in the softmax function~\citep{hinton2015distilling,pereyra2017regularizing} and adding small controlled perturbations to inputs,  the softmax score gap between in - and out-of-distribution examples is further enlarged. We  show that the combination of these two techniques (temperature scaling and input perturbation) can lead to better detection performance. For example, provided with a pre-trained DenseNet~\citep{huang2016densely} on CIFAR-10 dataset (positive samples), we test against images from TinyImageNet dataset (negative samples). Our method reduces the False Positive Rate (FPR), i.e., the fraction of misclassified out-of-distribution samples, from $34.7\%$ to $4.3\%$, when $95\%$ of in-distribution images are correctly classified. 
We summarize the main contributions of this paper as the following:
\begin{itemize}[leftmargin=*]
\item We propose a simple and effective method, ODIN (\textbf{O}ut-of-\textbf{DI}stribution detector for \textbf{N}eural networks), for detecting out-of-distribution examples in neural networks. Our method does not require re-training the neural network and is easily implementable on any modern neural architecture.
\item We test ODIN on state-of-the-art network architectures (e.g., DenseNet~\citep{huang2016densely} and Wide ResNet~\citep{wide}) under a diverse set of in- and out-distribution dataset pairs. We show ODIN can significantly improve the detection performance, and consistently outperforms the baseline method~\citep{hendrycks2016baseline} by a large margin. 
\item We empirically analyze how parameter settings affect the performance, and further provide simple analysis that provides some intuition behind our method. 
\end{itemize}

The outline of this paper is as follows. In Section~\ref{sec:problemformulation}, we present the necessary definitions and the problem statement.  In Section~\ref{sec:method}, we introduce ODIN and present performance results in Section~\ref{sec:experiment}. We experimentally analyze the proposed method and provide some justification for our method in Section~\ref{sec:discussions}. We summarize the related works and future directions in Section~\ref{sec::related} and conclude the paper in Section~\ref{sec:conclusions}.

\section{Problem Statement}\label{sec:problemformulation}
In this paper, we consider the problem of distinguishing in- and out-of-distribution images on a pre-trained neural network.  Let $P_{\bm{X}}$ and $Q_{\bm{X}}$ denote two distinct data distributions defined on the image space $\mathcal{X}$. Assume that a neural network $\bm{f}$ is trained on a dataset  drawn from the distribution $P_{\bm{X}}$.  We call $P_{\bm{X}}$  the \textbf{in-distribution} and $Q_{\bm{\bm{X}}}$ the  \textbf{out-distribution}, respectively. In testing, we draw new images from a mixture distribution $\mathbb{P}_{\bm{X}\times Z}$ defined on $\mathcal{X}\times\{0,1\}$, where the conditional probability distributions ${\mathbb{P}_{\bm{X}|Z=0}=P_{\bm{X}}}$ and $\mathbb{P}_{\bm{X}|Z=1}=Q_{\bm{X}}$ denote in- and out-distribution respectively. We consider the following problem: Given an image $\bm{X}$ drawn from the mixture distribution $\mathbb{P}_{\bm{X}\times Z}$, {\em can we distinguish whether the image is from in-distribution $P_{\bm{X}}$ or not}?

In this paper, we focus on detecting out-of-distribution images. However, it is equally important to correctly classify an image into the right class if it is an in-distribution image. But this can be easily done: once it has been detected that an image is in-distribution, we can simply use the original image
 and run it through the neural network to classify it. Thus, we do not change the predictions of the neural network for in-distribution images and only focus on improving the detection performance for out-of-distribution images.

\section{ODIN: Out-of-distribution Detector}\label{sec:method}


{In this section, we present our method, ODIN, for detecting out-of-distribution samples. The detector is built on two components: temperature scaling and input preprocessing. We describe the details of both components below.}

\para{Temperature Scaling.} Assume that the neural network $\bm{f}=(f_{1},...,f_{N})$ is trained to classify $N$ classes. For each input $\bm{x}$, the neural network assigns a label $\hat{y}(\bm{x})=\arg\max_{i}S_{i}(\bm{x};T)$  by computing the softmax output for each class. Specifically,\vspace{-0.1cm} 
\begin{equation}\label{eq::softmax}S_{i}(\bm{x};T)=\frac{\exp\left({{f_{i}(\bm{x})}/{T}}\right)}{\sum_{j=1}^{N}\exp\left({{f_{j}(\bm{x})}/{T}}\right)},\vspace{-0.2cm} \end{equation}
where $T\in\mathbb{R}^{+}$ is the temperature scaling parameter and  set to 1 during the training. For a given input $\bm{x}$, we call the maximum softmax probability, i.e., $S_{\hat{y}}(\bm{x};T)=\max_i S_{i}(\bm{x};T)$  the \textbf{softmax score}. In this paper, we use notations $S_{\hat{y}}(\bm{x};T)$ and $S(\bm{x};T)$ interchangeably. Prior works have established the use of temperature scaling to distill the knowledge in neural networks~\citep{hinton2015distilling} and calibrate the prediction confidence in classification tasks \citep{guo2017calibration}. As we shall see, using temperature scaling can  separate the softmax scores between in- and out-of-distribution images, making out-of-distribution detection effective. 


\para{Input Preprocessing.} In addition to temperature scaling, we preprocess the input by adding small perturbations:
\vspace{-0.1cm}
\begin{equation}\label{eq:perturbation}\tilde{\bm{x}}={\bm{x}}-\varepsilon\text{sign}(-\nabla_{{\bm{x}}}\log S_{\hat{y}}({\bm{x}};T)),\end{equation}
where the parameter $\varepsilon$ is the  perturbation magnitude. The method is inspired by the idea of adversarial examples ~\citep{goodfellow2014explaining}, where small perturbations are added to decrease the softmax score for the {true} label  and force the neural network to make a wrong prediction.  
Here, our goal and setting are the opposite: we aim to increase the softmax score of any given input, without the need for a class label at all. As we shall see later, the perturbation can have stronger effect on the in- distribution images than that on out-of-distribution images, making them more separable. 
 Note that the perturbations can be easily computed by back-propagating the gradient of the cross-entropy loss w.r.t the input.

\para{Out-of-distribution Detector.} The detector combines the two components described above. For each image $\bm{x}$, we first calculate  the preprocessed image $\tilde{\bm{x}}$ according to the equation~\eqref{eq:perturbation}. Next, we feed the preprocessed image $\tilde{\bm{x}}$ into the neural network, calculate its calibrated softmax score $S(\tilde{\bm{x}};T)$ and compare the score to the threshold $\delta$. An image $\bm{x}$ is classified as in-distribution if the softmax score  is greater than the threshold and vice versa. Mathematically, the out-of-distribution detector can be described as \vspace{-0.2cm}
$$g(\bm{x};\delta,T,\varepsilon)=\left\{\begin{matrix}1&\text{if} \max_{i}p(\tilde{\bm{x}};T)\le\delta,\\ 0&\text{if} \max_{i}p(\tilde{\bm{x}};T)>\delta. \end{matrix}\right.$$
The parameters $T,\varepsilon$ and $\delta$ are chosen so that the true positive rate (i.e., the fraction of in-distribution images correctly classified as in-distribution images) is $95\%.$ 

\section{Experiments}\label{sec:experiment}

In this section, we demonstrate the effectiveness of {ODIN} on several computer vision benchmark datasets.  
We run all experiments with PyTorch\footnote{\url{http://pytorch.org}} and we release the code to reproduce all {experimental} results\footnote{\url{https://github.com/facebookresearch/odin}}.

\subsection{Training Setup}\label{sec::training}

\para{Architectures and training configurations.} We adopt {two} state-of-the-art neural network architectures, including  {\em DenseNet}~\citep{huang2016densely} and {\em Wide ResNet}~\citep{wide}. 
For DenseNet, our model follows the same setup as in ~\citep{huang2016densely}, with depth $L=100$, growth rate $k=12$ (Dense-BC) and dropout rate 0. In addition, we evaluate the method on a Wide ResNet, with depth $28$, width 10 (WRN-28-10) and dropout rate 0. The hyper-parameters of neural networks are set identical to the original Wide ResNet~\citep{wide} and DenseNet~\citep{huang2016densely} implementations. All neural networks are trained with stochastic gradient descent with Nesterov momentum~\citep{duchi2011adaptive,kingma2014adam}. Specifically, we train Dense-BC for 300 epochs with batch size 64 and momentum 0.9;  
and Wide ResNet for 200 epochs with batch size 128 and momentum 0.9.  
The learning rate starts at 0.1, and is dropped by a factor of 10 at  $50\%$ and $75\%$ of the training progress, respectively. 

\begin{wraptable}{r}{0.4\textwidth}
\vspace{-0.5cm}
\centering
\small
\begin{tabular}{ccc}
\toprule
{Architecture} & {\bf C-10} & {\bf C-100} \\
\midrule
{\bf Dense-BC}& 4.81 & 22.37\\
{\bf WRN-28-10} & 3.71 & 19.86 \\
\bottomrule
\end{tabular}
\caption{Test error rates on {CIFAR-10}   and CIFAR-100  datasets.}
\vspace{-0.5cm}
\label{tab:pre-train-acc}
\end{wraptable}
\vspace{0.1cm}
\para{Accuracy of pre-trained networks.} Each neural network architecture is trained on CIFAR-10 (C-10)  and CIFAR-100 (C-100) datasets~\citep{cifar}, respectively. CIFAR-10 and CIFAR-100 images are drawn from 10 and 100 classes, respectively. Both datasets consist of 50,000 training images and 10,000 test images. The test error on CIFAR datasets are given in Table~\ref{tab:pre-train-acc}.

\subsection{Out-of-distribution Datasets}\label{sec::ood-datasets}
At test time, the test images from CIFAR-10 (CIFAR-100) datasets can be viewed as the in-distribution (positive) examples. For out-of-distribution (negative) examples, we follow the setting in~\citep{hendrycks2016baseline} and test on several different natural image datasets and synthetic noise datasets. We consider the following out-of-distribution test datasets. 
\begin{itemize}[leftmargin=*]
\item[(1)] \para{TinyImageNet.} The Tiny ImageNet dataset\footnote{\url{https://tiny-imagenet.herokuapp.com}} consists of a subset of ImageNet images~\citep{deng2009imagenet}. It contains 10,000 test images from 200 different classes. We construct two datasets, {\em TinyImageNet (crop)} and {\em TinyImageNet (resize)}, by either randomly cropping image patches of size $32 \times 32$ or downsampling each image to size $32 \times 32$. 

\item[(2)]\para{LSUN.} The Large-scale Scene UNderstanding dataset (LSUN) has a testing set of {10,000} images of 10 different scenes categories such as \emph{bedroom, kitchen room, living room, etc.}~\citep{yu15lsun}. Similar to TinyImageNet, we construct two datasets, {\em LSUN (crop)} and {\em LSUN (resize)}, by randomly cropping and downsampling the LSUN testing set, respectively. 


\item[(3)]\para{Gaussian Noise.} The synthetic Gaussian noise dataset consists of 10,000 random 2D Gaussian noise images, where each RGB value of every pixel is sampled from an i.i.d Gaussian distribution with mean 0.5 and unit variance. We further clip each pixel value into the range $[0, 1]$.  

\item[(4)]\para{Uniform Noise.} The synthetic uniform noise dataset consists of 10,000 images where each RGB value of every pixel is independently and identically sampled from a uniform distribution on $[0,1]$.  
\end{itemize}

For hyperparameter tuning, we use a separate validation dataset iSUN~\citep{xu2015turkergaze}, which is independent from the OOD test datasets. iSUN~\citep{xu2015turkergaze}  consists of natural scene images. We include the entire collection of 8925 images in iSUN and downsample each image to size $32$ by $32$.

\begin{table}[t]
\vspace{-0.5cm}
\centering
\small
\begin{tabular}{llcccccc}
\toprule
& {\bf Out-of-distribution}  &\bf{FPR} & \bf{Detection} & {\bf AUROC} & {\bf AUPR} &   {\bf AUPR}\\
& {\bf dataset} &\bf{(95\% TPR)} & \bf{Error}& &\bf{In}  &\bf{Out}\\
&   &$\downarrow$ & $\downarrow$&  $\uparrow$ & $\uparrow$ &$\uparrow$ \\
\midrule
\multirow{8}{0.12\linewidth}{{{\bf Dense-BC} CIFAR-10}}  
&   & \multicolumn{5}{c}{{ \bf{Baseline~\citep{hendrycks2016baseline} / ODIN}}} \\
\cmidrule{3-7}
& TinyImageNet (crop) &34.7/\bf{4.3} &10.0/\bf{4.7}& 95.3/\bf{99.1}& 96.4/\bf{99.1}&93.8/\bf{99.1}  \\ 
& TinyImageNet (resize) &40.8/\bf{7.5} & 11.5/\bf{6.1}& 94.1/\bf{98.5}& 95.1/\bf{98.6}&92.4/\bf{98.5}  \\ 
& LSUN (crop) &39.3/\bf{11.4} & 10.2/\bf{7.2}& 94.8/\bf{97.9}& 96.0/\bf{98.0}&93.1/\bf{97.9}  \\ 
& LSUN (resize) &33.6/\bf{3.8} & 9.8/\bf{4.4}& 95.4/\bf{99.2}& 96.4/\bf{99.3}&94.0/\bf{99.2}  \\ 
& Uniform &23.5/\bf{0.0} & 5.3/\bf{0.5}& 96.5/\bf{99.0}& 97.8/\bf{100.0}&93.0/\bf{99.0}  \\ 
& Gaussian &12.3\bf{/0.0} &4.7/\bf{0.2}& 97.5/\bf{100.0}& 98.3/\bf{100.0}&95.9/\bf{100.0}  \\ 
\midrule
\multirow{7}{0.12\linewidth}{{ {\bf Dense-BC} CIFAR-100}}  
& TinyImageNet (crop) &67.8/\bf{26.9}& 36.4/\bf{12.9}& 83.0/\bf{94.5}& 85.3/\bf{94.7}&80.8/\bf{94.5}  \\ 
& TinyImageNet (resize) &82.2/\bf{57.0} & 43.6/\bf{22.7}& 70.4/\bf{85.5}& 71.4/\bf{86.0}&68.6/\bf{84.8}  \\ 
& LSUN (crop) &69.4/\bf{18.6} & 37.2/\bf{9.7}& 83.7/\bf{96.6}& 86.2/\bf{96.8}&80.9/\bf{96.5}  \\ 
& LSUN (resize) &83.3/\bf{58.0} & 44.1/\bf{22.3}& 70.6/\bf{86.0}& 72.5/\bf{87.1}&68.0/\bf{84.8}  \\ 
& Uniform &100.0/\bf{100.0}& 35.86/\bf{17.9}& 43.1/\bf{99.5}& 63.2/\bf{87.5}&41.9/\bf{65.1}  \\   
& Gaussian &100.0/\bf{100.0} & 41.2/\bf{38.0}& 30.6/\bf{40.5}& 53.4/\bf{60.5}&37.6/\bf{40.9}  \\ 
\bottomrule
\end{tabular}
\vspace{-0.2cm}
\caption[]{\small Distinguishing in- and out-of-distribution test set data for image classification. All values are percentages. $\uparrow$ indicates larger value is better, and $\downarrow$ indicates lower value is better. We use $T=1000$ for all experiments. The noise magnitude $\varepsilon$ was selected on a \textbf{separate validation dataset}, which is different from the out-of-distribution test sets. On CIFAR-10 pretrained model, we use $\varepsilon=0.0014$ for all OOD test datasets; and $\varepsilon=0.002$ for CIFAR-100 pretrained model.}
\label{tab:main-results}
\vspace{-0.5cm}
\end{table}

\subsection{Evaluation metrics}
We adopt the following four different metrics to measure the effectiveness of a neural network in distinguishing in- and out-of-distribution images. 
\begin{itemize}[leftmargin=*]
\item[(1)] {\bf FPR at $95\%$ TPR } can be interpreted as the probability that a negative (out-of-distribution) example is misclassified as positive (in-distribution) when the true positive rate (TPR) is as high as $95\%$. 
\item[(2)] {\bf Detection Error}, i.e., $P_{e}$  measures the misclassification probability when TPR is 95\%. The definition of $P_{e}$ is given by $P_e=0.5(1-\text{TPR})+0.5\text{FPR}$, where we assume that both positive and negative examples have the equal probability of appearing in the test set. 
\item[(3)] {\bf AUROC} is the Area Under the Receiver Operating Characteristic curve, which is also a threshold-independent metric~\citep{davis2006relationship}. The ROC curve depicts the relationship between {TPR and FPR}. The AUROC can be interpreted as the probability that a positive example is assigned a higher detection score than a negative example~\citep{fawcett2006introduction}.  
A perfect detector corresponds to an AUROC score of $100\%$. 
\item[(4)] {\bf AUPR} is the Area under the Precision-Recall curve, which is another  threshold independent metric~\citep{manning1999foundations,saito2015precision}. The PR curve is a graph showing the precision=TP/(TP+FP) and {recall=TP/(TP+FN)} against  each other. The metric AUPR-In and AUPR-Out in Table~\ref{tab:main-results} denote the area under the precision-recall curve where in-distribution and out-of-distribution images are specified as positives, respectively. 
\end{itemize}

\begin{wrapfigure}{R}{0.35\linewidth}
\vspace{-0cm}
   \centering
    \includegraphics[width=1\linewidth]{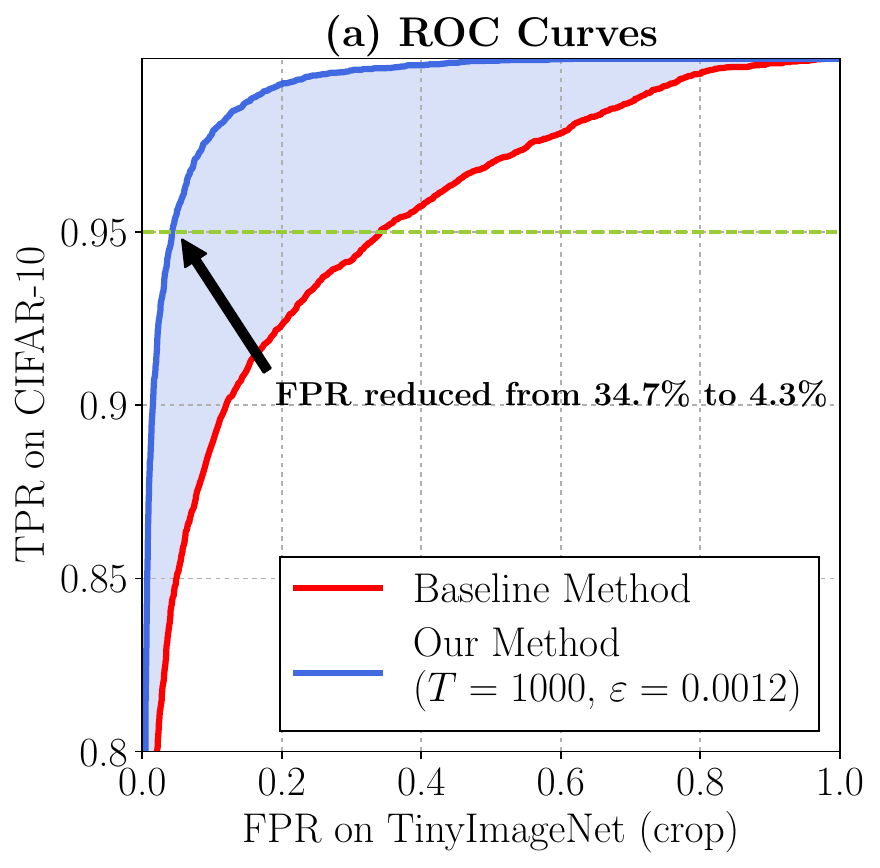}

  \vspace{-0.3cm}
  \caption{\small (a) ROC curves of baseline (red) and our method (blue) on DenseNet-BC-100 network, where CIFAR-10 and TinyImageNet (crop) are in- and out-of-distribution dataset, respectively.}
   \label{fig::original}
  \vspace{-0.5cm}
\end{wrapfigure}

\subsection{Experimental Results}\label{sec::expresults}

\para{Comparison with baseline.} In Figure~\ref{fig::original}, we show the ROC curves when  DenseNet-BC-100 is evaluated on CIFAR-10 (positive) images against TinyImageNet (negative) test examples.  The red curve corresponds to the ROC curve {when using} baseline method~\citep{hendrycks2016baseline}, whereas the blue curve corresponds to ODIN. We observe a {strikingly large} gap between the blue and red ROC curves. For example, when TPR$=95\%$, the FPR can be reduced from $34\%$ to $4.2\%$ by {using our approach}.

\para{ Hyperparameters.} We use a separate OOD validation dataset for hyperparameter selection, which is independent from the OOD test datasets. For temperature $T$, we select among 1, 2, 5, 10, 20, 50, 100, 200, 500, 1000; and for perturbation magnitude $\varepsilon$ we choose from 21 evenly spaced numbers starting from 0 and ending at 0.004. The optimal parameters are chosen to minimize the FPR at TPR 95\% on the validation OOD dataset.  

\vspace{0.1cm}

\para{Main results.} The main results are summarized in  Table~\ref{tab:main-results}, where we use iSUN~\citep{xu2015turkergaze} as validation set. We use $T=1000$ for all settings. For DenseNet, we use $\varepsilon=0.0014$ for CIFAR-10 and $\varepsilon=0.002$ for CIFAR-100. We provide additional details on the effect of parameters in Section~\ref{sec:discussions}.  For each in- and out-of-distribution dataset pair, we report both the performance of the baseline~\citep{hendrycks2016baseline} and ODIN. In Table~\ref{tab:main-results}, we observe significant performance  improvement across all dataset pairs. 

%


\para{Parameter transferability.} In Table~\ref{tab::add1}, we show how the parameters tuned on one validation set can generalize across datasets. Specifically, we tune the parameters using one validation dataset and then evaluated on the remaining OOD test datasets. The results are very similar across different validation sets, which suggests the insensitivity of our method w.r.t the tuning set.

\begin{table}[h]
\centering
\small
\begin{tabular}{lcccccccc}
\toprule
	&\multicolumn{7}{c}{ \bf{DenseNet-BC-100}} \\
\midrule
 {\bf{Validation set}} &{\bf  ImgNet (c)} & {\bf  ImgNet (r)} &  {\bf  LSUN (c)} &  {\bf  LSUN (r)} &  {\bf  iSUN} &{\bf  Gaussian} &{\bf  Uniform}  \\
 \midrule
  \textbf{Test set} & \multicolumn{7}{c}{{ \bf{Baseline~\citep{hendrycks2016baseline} / ODIN}}} \\
\cmidrule{2-8}
 {\bf  ImgNet (c)}  	&-	& 34.7/4.3		&34.7/6.6	  & 34.7/4.3		&34.7/4.3		        & 34.7/4.3		&34.7/4.3	      \\
 {\bf  ImgNet (r)}  	&40.7/7.5	 &-		&40.7/14.9          &40.7/7.5	        &40.7/7.5		&40.7/7.5         &40.7/7.5	  \\
 {\bf  LSUN (c)}  	&39.3/13.8 &39.3/13.8  &-           &39.3/13.8         &39.3/11.4       &39.3/13.8    &39.3/13.8\\
 {\bf  LSUN (r)}  	&33.6/4.8   &33.6/4.8     &33.6/10.4           &-            &33.6/3.8          &33.6/4.8       &33.6/4.8    \\
 {\bf  Gaussian}         &23.5/0.0    &23.5/0.0    &23.5/0.4    &23.5/0.0    &23.5/0.0      &-    &23.5/0.0 \\
 {\bf  Uniform}  	        &12.3/0.0     &12.3/0.0     &12.3/4.5   &12.3/0.0     &12.3/0.0   &12.3/0.0     &-  \\
\bottomrule
\end{tabular}
\caption[]{Detection performance using different validation OOD datasets. The hyperparameters  are tuned using one validation dataset and then evaluate on the remaining OOD test datasets. The neural network is pre-trained on CIFAR-10.}
\label{tab::add1}
\end{table}

\para{Data distributional distance vs. detection performance.} To measure the statistical distance between in- and out-of-distribution datasets, we adopt a commonly used metric, maximum mean discrepancy (MMD) with Gaussian RBF kernel~\citep{sriperumbudur2010hilbert,gretton2012kernel,sutherland2016generative}. Specifically, given two image sets, $V=\{v_{1},...,v_{m}\}$ and $W=\{w_{1},...,w_{m}\}$, the maximum mean discrepancy between $V$ and $Q$ is  defined as 
$$\widehat{\text{MMD}}^{2}(V,W)=\frac{1}{{{m}\choose{2}}}\sum_{i\neq j}k(v_{i},v_{j})+\frac{1}{{{m}\choose{2}}}\sum_{i\neq j}k(w_{i},w_{j})-\frac{2}{{{m}\choose{2}}}\sum_{i\neq j}k(v_{i},w_{j}),$$
 where $k(\cdot, \cdot)$ is the Gaussian RBF kernel, i.e., $k(x, x')=\exp\left(-\frac{\|x-x'\|_{2}^{2}}{2\sigma^{2}}\right)$. We use the same method used by ~\cite{sutherland2016generative} to choose $\sigma$,  where $2\sigma^{2}$ is set to the median of all Euclidean distances between all images in the aggregate set $V\cup W$.

\begin{figure}
   \centering
    \includegraphics[width=0.6\textwidth]{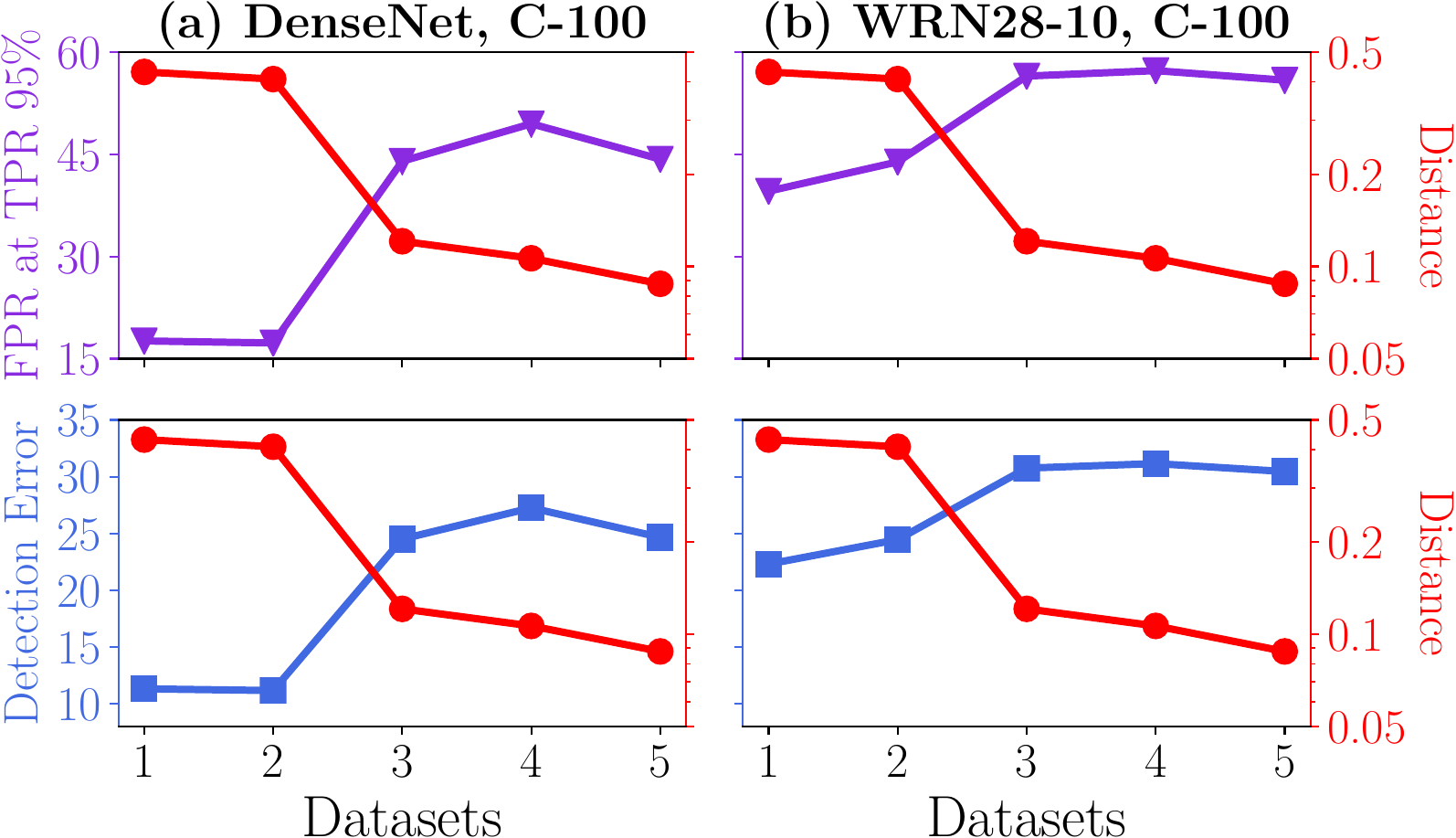}

  \vspace{-0.3cm}
  \caption{\small  (a)-(b) Performance of our method vs. MMD between in- and out-of-distribution datasets. Neural networks are trained on CIFAR-100. The out-of-distribution datasets are 1: LSUN (cop), 2: TinyImageNet (crop), 3: LSUN (resize), 4: is iSUN (resize), 5: TinyImageNet (resize).}
   \label{fig::mmd}
\end{figure}

In Figure~\ref{fig::mmd} (a)(b), we show how the performance of ODIN varies against the MMD distances between in- and out-of-distribution datasets. The datasets (on x-axis) are ranked in the descending order of MMD distances with CIFAR-100. There are two interesting observations can be drawn from these figures. First, we find that the MMD distances between the cropped datasets and CIFAR-100 tend to be larger. This is likely due to the fact that cropped images only contain local image context and are therefore more distinct from CIFAR-100 images, while resized images contain global patterns and are thus similar to images in CIFAR-100. Second, we observe that the MMD distance tends to be negatively correlated with the detection performance. This suggests that the detection task becomes harder as in- and out-of-distribution images are more similar to each other.



\section{Discussions}\label{sec:discussions}



\subsection{Analysis on Temperature Scaling}\label{sec::temperature}

In this subsection, we analyze the effectiveness of the temperature scaling method. As shown in Figure~\ref{fig::results} (a) and (b), we observe that a sufficiently large temperature yields better detection performance although the effects diminish when $T$ is too large. To gain insight, we can use the Taylor expansion of the softmax score (details provided in Appendix~\ref{appendix::taylor}). When $T$ is sufficiently large, we have \vspace{-0.3cm}
\begin{equation}\label{eq::softapprox}
S_{\hat{y}}(\bm{x};T)\approx\frac{1}{N-\frac{1}{T}\sum_{i}[f_{\hat{y}}(\bm{x})-f_{i}(\bm{x})]+\frac{1}{2T^{2}}\sum_{i}[f_{\hat{y}}(\bm{x})-f_{i}(\bm{x})]^{2}}, 
\end{equation}
by omitting the third and higher orders.  
For simplicity of notation,  we define  
\begin{equation}\label{eq::u1u2}
U_{1}(\bm{x})=\frac{1}{N-1}\sum_{i\neq\hat{y}}[f_{\hat{y}}(\bm{x})-f_{i}(\bm{x})]\quad\text{and}\quad U_{2}(\bm{x})=\frac{1}{N-1}\sum_{i\neq\hat{y}}[f_{\hat{y}}(\bm{x})-f_{i}(\bm{x})]^{2}. \vspace{-0.3cm}
\end{equation}

 
 \begin{figure}[t]
	\centering
	\includegraphics[width=\linewidth]{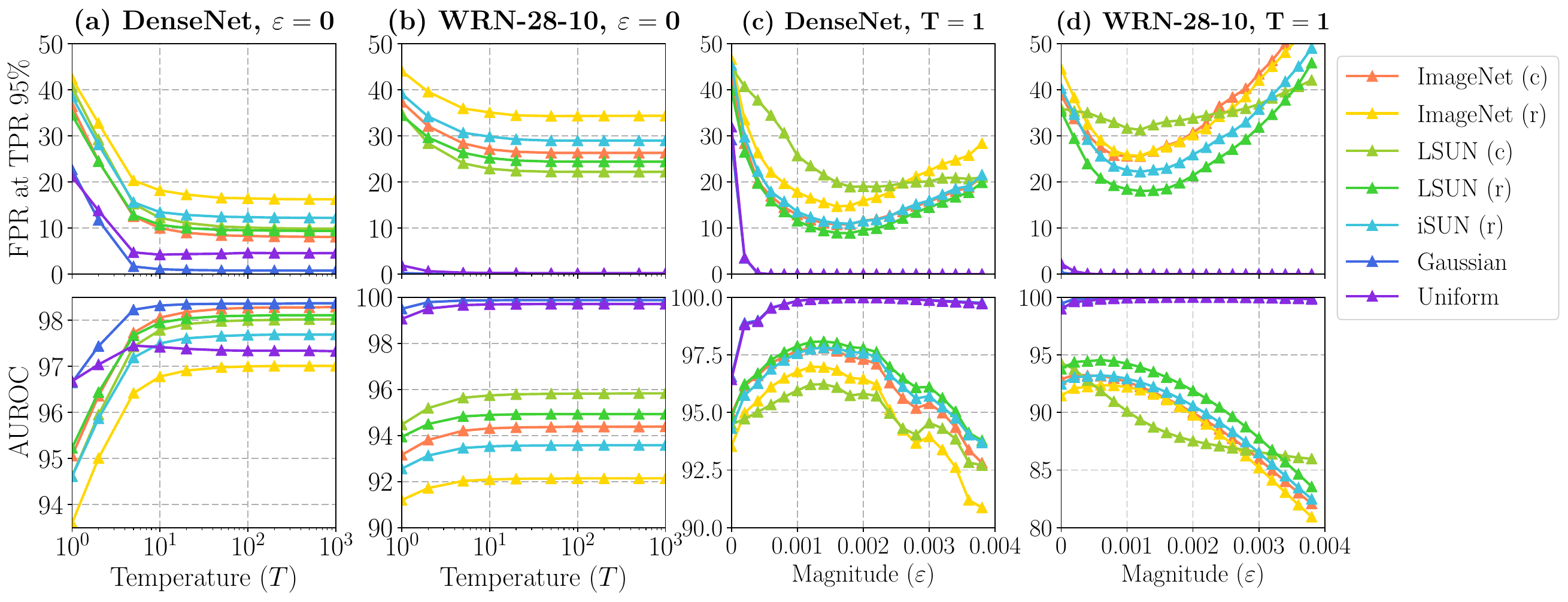}\vspace{-0.2cm}
	\caption{\small (a)(b) Effects of temperature $T$ when $\varepsilon=0$.  (c)(d) Effects of perturbation magnitude $\varepsilon$ when $T=1$. All networks are trained on CIFAR-10 (in-distribution).} 
	\label{fig::results}
	\vspace{-0.3cm}
\end{figure}

\begin{figure}[t]
   \centering
    \includegraphics[width=1\textwidth]{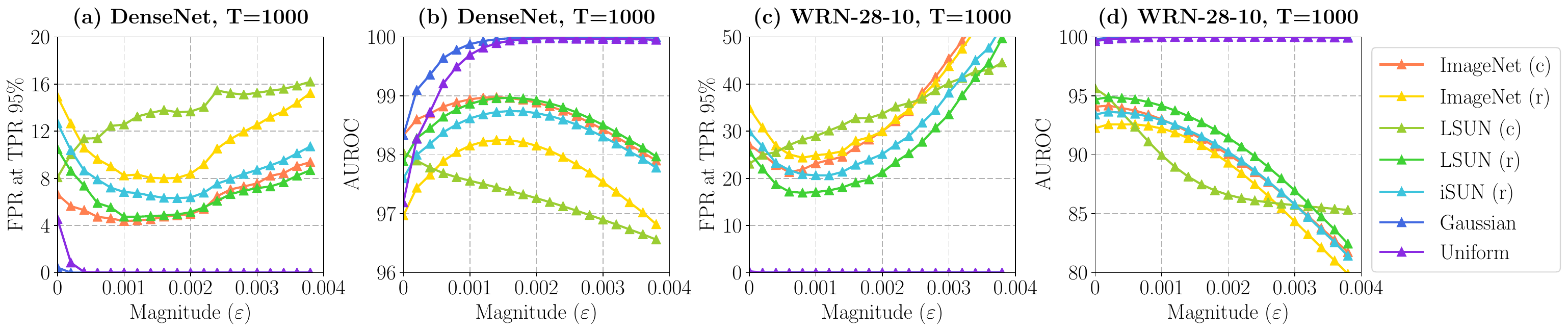}

  \vspace{-0.2cm}
  \caption{\small (a)(b) Effects of perturbation magnitude $\varepsilon$ on DenseNet when $T$ is large (e.g., $T=1000$).  (c)(d)  Effects of perturbation magnitude of $\varepsilon$ on Wide-ResNet-28-10 when $T$ is large (e.g., $T=1000$). All networks are trained on CIFAR-10. }
   \label{fig::param2}
  \vspace{-0.3cm}
\end{figure}

\para{Interpretations of $U_{1}$ and $U_{2}$}. By definition, $U_{1}$ measures the extent to which the largest unnormalized output of the neural network deviates from the remaining outputs; while $U_{2}$ measures the extent to which the remaining smaller outputs deviate from each other. We provide formal mathematical derivations in Appendix~\ref{appendix::variance}. In Figure~\ref{fig::analysis}(a), we show the distribution of $U_{1}$ for each out-of-distribution dataset vs. the in-distribution dataset (in red). We observe that the largest outputs of the neural network on in-distribution images deviate more from the remaining outputs. This is likely due to the fact that neural networks tend to make more confident predictions on in-distribution images.

\begin{figure}[]
 \vspace{-0.5cm}
  \includegraphics[width=\linewidth]{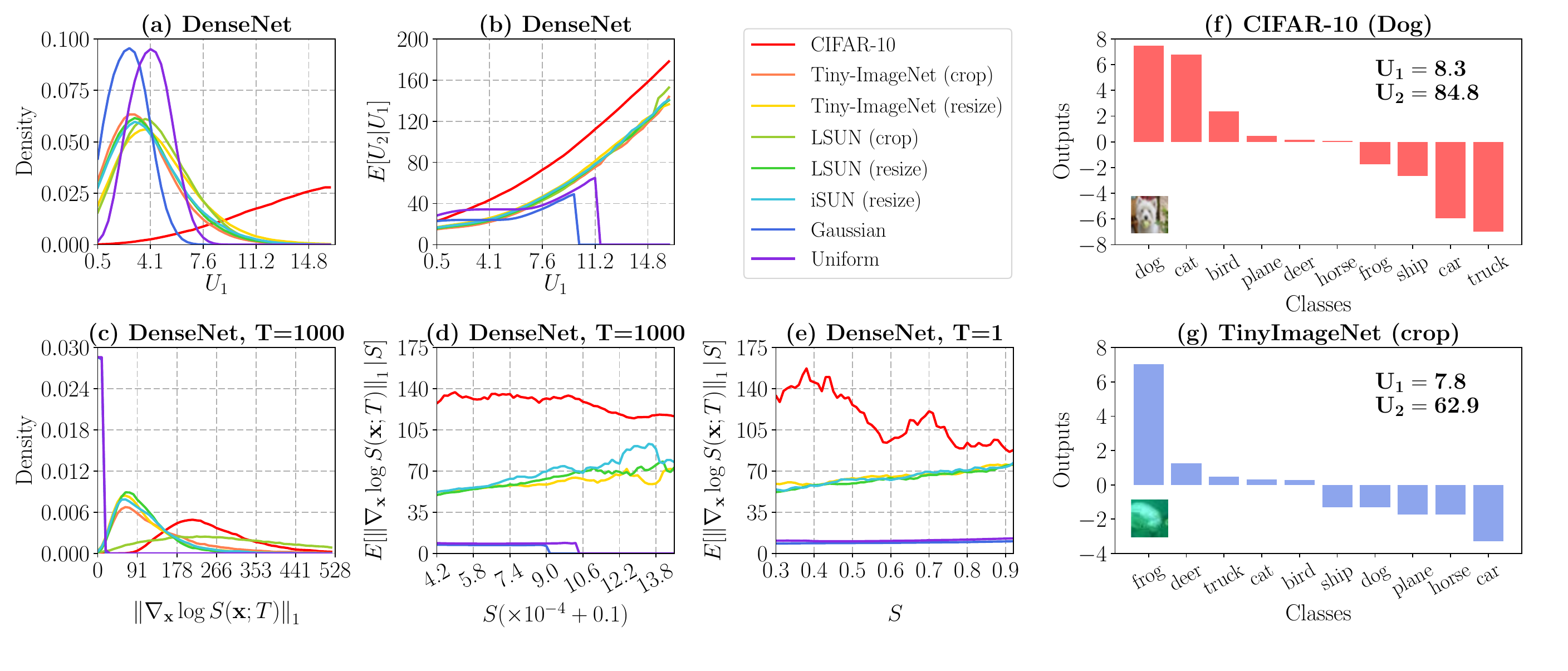}
  \vspace{-1cm}
  \caption{\small (a) Probability density of $U_{1}$ under different datasets on DenseNet.
(b) Expectations of  $U_{2}$ conditioned on  $U_{1}$ on DenseNet.
(c) Probability density of the norm of gradient on DenseNet under temperature $1,000$. 
(c)(d) Expectation of the norm of gradient conditioned on the softmax scores on DenseNet under temperature  $T=1000$ and $T=1$, respectively. 
(f)(g) Outputs of DenseNet on each class  for an image of dog from CIFAR-10 and an image from TinyImageNet (crop). The DenseNet is trained on CIFAR-10.
 Additional results on other architectures are provided  in Appendix~\ref{appendix::analysis}.}
  \label{fig::analysis}
  \vspace{-0.5cm}
\end{figure}

Further, we show  in Figure~\ref{fig::analysis}(b) the expectation of $U_{2}$ conditioned on $U_{1}$, i.e., $E[U_{2}|U_{1}]$, for each dataset. The red curve (in-distribution images) has overall higher expectation. This indicates that, when two images have similar values on $U_{1}$, the in-distribution image tends to have a much higher value of $U_{2}$ than the out-of-distribution image. In other words, for in-distribution images, the remaining outputs (excluding the largest output) tend to be more separated from each other compared to out-of-distribution datasets. This may happen when some classes in the in-distribution dataset share common features while others differ significantly. To illustrate this, in Figure~\ref{fig::analysis} (f)(g), we show the outputs of each class using a DenseNet (trained on CIFAR-10) on a dog image from CIFAR-10, and another image from TinyImageNet (crop). For the image of dog, we can observe that the largest output for the label {\em dog} is close to the output for the label {\em cat} but is quite separated from the outputs for  the label {\em car} and {\em truck}. This is likely due to the fact that, in CIFAR-10, images of dogs are very similar to the images of cats but are quite distinct from images of car and truck. For the image from TinyImageNet (crop), despite having one large output, the remaining outputs are close to each other and thus have a smaller deviation.

\para{The effects of $T$.} To see the usefulness of adopting a large $T$, we can first rewrite the softmax score function in Equation~(\ref{eq::softapprox}) as  ${S\propto {(U_{1}-U_{2}/2T)/T}}$.  Hence the softmax score is largely determined by $U_1$ and $U_2/2T$. As noted earlier, $U_{1}$ makes in-distribution images produce larger softmax scores than out-of-distribution images since $S\propto U_{1}$, while $U_{2}$ has the exact opposite effect since $S\propto -U_{2}$. Therefore, by choosing a sufficiently large temperature, we can compensate the negative impacts of $U_{2}/2T$ on the detection performance, making the softmax scores between in- and out-of-distribution images more separable. Eventually, when $T$ is sufficiently large, the distribution of softmax score is almost dominated by the distribution of $U_{1}$ and thus increasing the temperature further is no longer effective. This explains why we see in Figure~\ref{fig::results} (a)(b) that the performance does not change when $T$ is too large (e.g., $T>100$). In Appendix~\ref{appendix::prop1}, we provide a formal proof showing that the detection error eventually converges to a constant number when $T$ goes to infinity. 

\vspace{-0.2cm}
\subsection{Analysis on Input Preprocessing}\label{sec::preprocessing}
\vspace{-0.2cm}
  
As noted previously, using the temperature scaling method by itself can be effective in improving the detection performance. However, the effectiveness quickly diminishes as $T$ becomes very large. In order to make further improvement, we complement temperature scaling with input preprocessing. This has already been seen in Figure~\ref{fig::param2}, where the detection performance is improved by a large margin on most datasets when $T=1000$, provided with an appropriate perturbation magnitude $\varepsilon$ is chosen. In this subsection, we provide some intuition behind this. 

 To explain, we can look into the first order Taylor expansion of the log-softmax function for the perturbed image $\tilde{\bm{x}}$, which is given by\vspace{-0cm}
\begin{equation*}
\log S_{\hat{y}}(\tilde{\bm{x}};T)=\log S_{\hat{y}}(\bm{x};T)+ \varepsilon\left\|\nabla_{\bm{x}}\log S_{\hat{y}}(\bm{x};T)\right\|_{1}+o(\varepsilon),
\end{equation*}
where ${\bm{x}}$ is the original input. 

\begin{wrapfigure}{R}{0.3\linewidth}
\centering
 \vspace{-0.7cm}
  \includegraphics[width=1\linewidth]{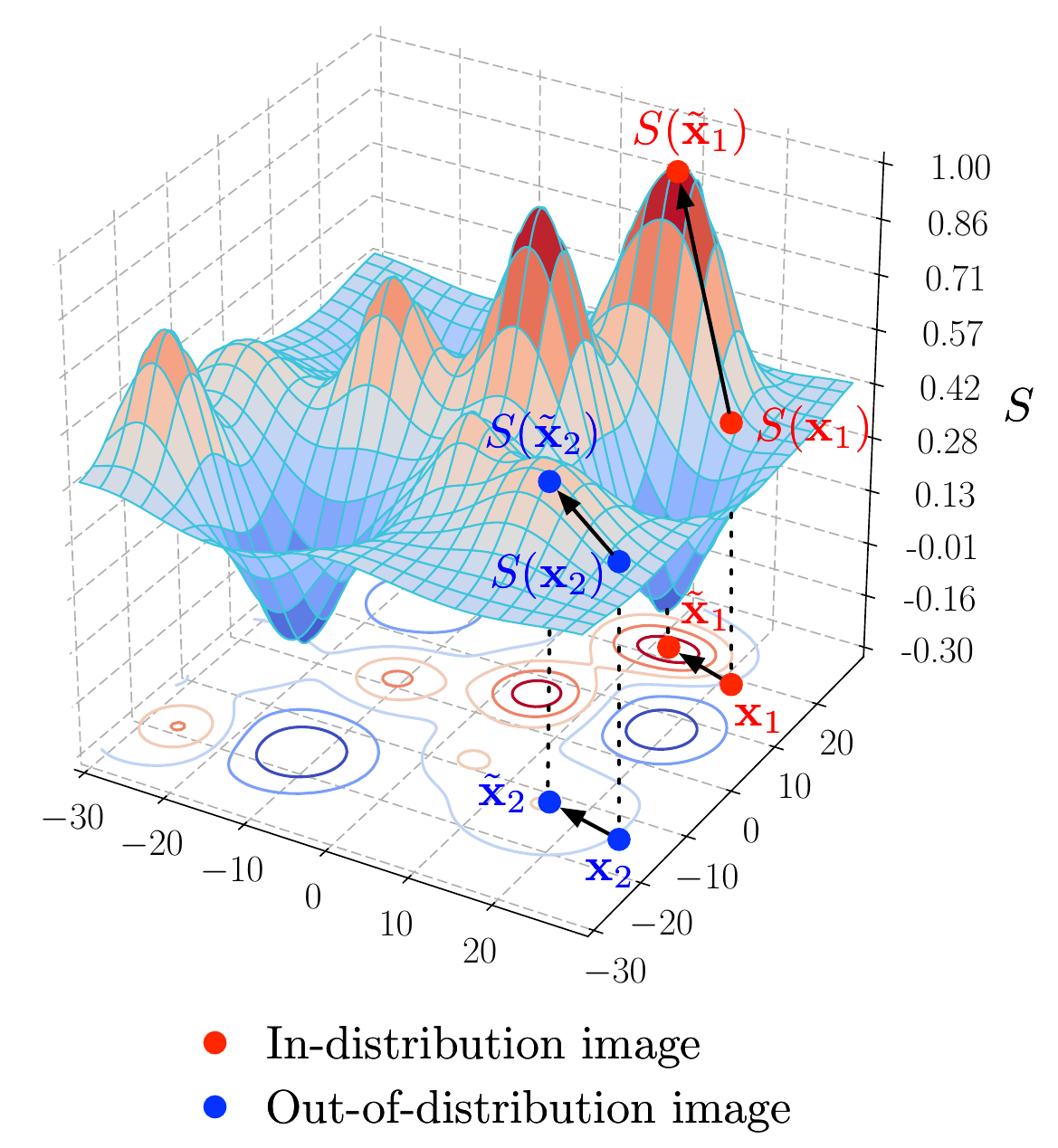}
  \vspace{-0.5cm}
  \caption{\small Illustration of effects of the input preprocessing.}
  \label{fig::gradient}
  \vspace{-0.3cm}
  \end{wrapfigure}

\para{The effects of gradient. } In Figure~\ref{fig::analysis} (c), we present the distribution of $\|\nabla_{\bm{x}}\log S(\bm{x};T)\|_{1}$ --- the {1-norm} of gradient of log-softmax with respect to the input $\bm{x}$ --- for all datasets. A salient observation is that CIFAR-10 images (in-distribution) tend to have larger values on the norm of gradient than most out-of-distribution images. To further see the effects of the norm of gradient on the softmax score, we provide in Figures~\ref{fig::analysis}~(d) the conditional expectation $E[\|\nabla_{\bm{x}}\log S(\bm{x};T)\|_{1}|S]$. We can observe that, when an in-distribution image and an out-of-distribution image have the same softmax score, the value of $\left\|\nabla_{\bm{x}}\log S(\bm{x};T)\right\|_{1}$ for in-distribution image tends to be larger. 

We illustrate the effects of the norm of gradient in Figure~\ref{fig::gradient}. Suppose that an in-distribution image $\bm{x}_{1}$ (blue) and  an out-of-distribution image $\bm{x}_{2}$ (red) have similar softmax scores, i.e.,  $S(\bm{x}_{1})\approx S(\bm{x}_{2})$. After input processing, the in-distribution image can have a much  larger softmax score than the out-of-distribution image $\bm{x}_{2}$ since $\bm{x}_{1}$ results in  a much larger value on the norm of softmax gradient than that of $\bm{x}_{2}$. Therefore, in- and out-of-distribution images are more separable from each other after input preprocessing\footnote{Similar observation can be seen when $T=1$, where we present the conditional expectation of the norm of softmax gradient in Figure~\ref{fig::analysis} (e).}.

\para{The effect of $\varepsilon$.} When the magnitude $\varepsilon$ is sufficiently small, adding perturbations  does not change the predictions of the neural network, i.e., $\hat{y}(\tilde{\bm{x}})=\hat{y}(\bm{x})$.  However, when $\varepsilon$ is not negligible, the gap of softmax scores between in- and out-of-distribution images can be affected by $\|\nabla_{\bm{x}}\log S(\bm{x};T)\|_{1}$. Our observation is  consistent with that in \citep{szegedy2013intriguing,goodfellow2014explaining,moosavi2016universal}, which show that the softmax scores tend to change significantly if small perturbations are added to the in-distribution images. It is also worth noting that using a very large $\varepsilon$ can lead to performance degradation, as seen in Figure~\ref{fig::param2}. This is likely due to the fact that the second and higher order terms in the Taylor expansion are no longer insignificant when the perturbation magnitude is too large.

\vspace{-0.2cm}
\section{Related Works and Future Directions}\label{sec::related}
\vspace{-0.2cm}
The problem of detecting out-of-distribution examples in low-dimensional space has been well-studied in various contexts (see the survey by \cite{pimentel2014review}).  Conventional methods such as density estimation, nearest neighbor and clustering analysis are widely used in detecting low-dimensional out-of-distribution examples~\citep{chow1970optimum, vincent2003manifold, ghoting2008fast, devroye2013probabilistic}, . The density estimation approach  uses probabilistic models to estimate the in-distribution density and declares a test example to be out-of-distribution if it locates in the low-density areas. The clustering method is based on the statistical distance, and declares an example to be out-of-distribution if it locates far from its neighborhood. Despite various applications in low-dimensional spaces, unfortunately, these methods are known to be  unreliable in high-dimensional space such as image space \citep{nonparametric,theis2015note}. In recent years, out-of-distribution detectors based on deep models have been proposed. \cite{schlegl2017unsupervised} train a generative adversarial networks to detect out-of-distribution examples in clinical scenario. \cite{sabokrou2016fully} train a  convolutional network to detect anomaly in scenes. \cite{andrews2016transfer} adopt transfer representation-learning for anomaly detection. All these works require enlarging or modifying the neural networks.  In a more recent work, \cite{hendrycks2016baseline} found that pre-trained neural networks can be overconfident to out-of-distribution example, limiting the effectiveness of detection. Our paper aims to improve the performance of detecting out-of-distribution examples, without requiring any change to an existing well-trained model. 


Our approach leverages the following two interesting observations to help better distinguish between in- and out-of-distribution examples: (1) On in-distribution images, modern neural networks tend to produce outputs with larger variance across class labels, and (2) neural networks have larger norm of gradient of log-softmax scores when applied on in-distribution images. 
We believe that having a better understanding of these phenomenon can lead to further insights into this problem. 


\vspace{-0.2cm}
\section{Conclusions}
\label{sec:conclusions}
\vspace{-0.2cm}

In this paper, we propose a simple and effective method to detect out-of-distribution data samples in neural networks. Our method does not require retraining the neural network and significantly improves on the baseline method~\cite{hendrycks2016baseline} on different neural  architectures across  various in and out-distribution dataset pairs. We empirically analyze the method under different parameter settings, and provide some  insights behind the approach. Future work involves exploring our method in other applications such as speech recognition and natural language processing. 

\section*{Acknowledgments}
The research reported here was supported by NSF Grant CPS ECCS 1739189.

\bibliography{iclr2018_conference}

\begin{thebibliography}{41}
\providecommand{\natexlab}[1]{#1}
\providecommand{\url}[1]{\texttt{#1}}
\expandafter\ifx\csname urlstyle\endcsname\relax
  \providecommand{\doi}[1]{doi: #1}\else
  \providecommand{\doi}{doi: \begingroup \urlstyle{rm}\Url}\fi

\bibitem[Amodei et~al.(2016)Amodei, Olah, Steinhardt, Christiano, Schulman, and
  Man{\'e}]{amodei2016concrete}
Dario Amodei, Chris Olah, Jacob Steinhardt, Paul Christiano, John Schulman, and
  Dan Man{\'e}.
\newblock Concrete problems in ai safety.
\newblock \emph{arXiv preprint arXiv:1606.06565}, 2016.

\bibitem[Andrews et~al.(2016)Andrews, Tanay, J.~Morton, and
  D.~Griffin]{andrews2016transfer}
Jerone~T.A Andrews, Thomas Tanay, Edward J.~Morton, and Lewis D.~Griffin.
\newblock Transfer representation-learning for anomaly detection.
\newblock In \emph{ICML}, 2016.

\bibitem[Cho et~al.(2014)Cho, Van~Merri{\"e}nboer, Gulcehre, Bahdanau,
  Bougares, Schwenk, and Bengio]{cho2014learning}
Kyunghyun Cho, Bart Van~Merri{\"e}nboer, Caglar Gulcehre, Dzmitry Bahdanau,
  Fethi Bougares, Holger Schwenk, and Yoshua Bengio.
\newblock Learning phrase representations using rnn encoder-decoder for
  statistical machine translation.
\newblock \emph{EMNLP}, 2014.

\bibitem[Chow(1970)]{chow1970optimum}
C~Chow.
\newblock On optimum recognition error and reject tradeoff.
\newblock \emph{IEEE Transactions on information theory}, 16\penalty0
  (1):\penalty0 41--46, 1970.

\bibitem[Davis \& Goadrich(2006)Davis and Goadrich]{davis2006relationship}
Jesse Davis and Mark Goadrich.
\newblock The relationship between precision-recall and roc curves.
\newblock In \emph{ICML}. ACM, 2006.

\bibitem[Deng et~al.(2009)Deng, Dong, Socher, Li, Li, and
  Fei-Fei]{deng2009imagenet}
Jia Deng, Wei Dong, Richard Socher, Li-Jia Li, Kai Li, and Li~Fei-Fei.
\newblock Imagenet: A large-scale hierarchical image database.
\newblock In \emph{CVPR}, 2009.

\bibitem[Devroye et~al.(2013)Devroye, Gy{\"o}rfi, and
  Lugosi]{devroye2013probabilistic}
Luc Devroye, L{\'a}szl{\'o} Gy{\"o}rfi, and G{\'a}bor Lugosi.
\newblock \emph{A probabilistic theory of pattern recognition}, volume~31.
\newblock Springer Science \& Business Media, 2013.

\bibitem[Duchi et~al.(2011)Duchi, Hazan, and Singer]{duchi2011adaptive}
John Duchi, Elad Hazan, and Yoram Singer.
\newblock Adaptive subgradient methods for online learning and stochastic
  optimization.
\newblock \emph{Journal of Machine Learning Research}, 12\penalty0
  (Jul):\penalty0 2121--2159, 2011.

\bibitem[Farabet et~al.(2013)Farabet, Couprie, Najman, and
  LeCun]{farabet2013learning}
Clement Farabet, Camille Couprie, Laurent Najman, and Yann LeCun.
\newblock Learning hierarchical features for scene labeling.
\newblock \emph{IEEE transactions on pattern analysis and machine
  intelligence}, 35\penalty0 (8):\penalty0 1915--1929, 2013.

\bibitem[Fawcett(2006)]{fawcett2006introduction}
Tom Fawcett.
\newblock An introduction to roc analysis.
\newblock \emph{Pattern recognition letters}, 2006.

\bibitem[Ghoting et~al.(2008)Ghoting, Parthasarathy, and Otey]{ghoting2008fast}
Amol Ghoting, Srinivasan Parthasarathy, and Matthew~Eric Otey.
\newblock Fast mining of distance-based outliers in high-dimensional datasets.
\newblock \emph{Data Mining and Knowledge Discovery}, 16\penalty0 (3):\penalty0
  349--364, 2008.

\bibitem[Goodfellow et~al.(2015)Goodfellow, Shlens, and
  Szegedy]{goodfellow2014explaining}
Ian~J Goodfellow, Jonathon Shlens, and Christian Szegedy.
\newblock Explaining and harnessing adversarial examples.
\newblock \emph{ICLR}, 2015.

\bibitem[Gretton et~al.(2012)Gretton, Borgwardt, Rasch, Sch{\"o}lkopf, and
  Smola]{gretton2012kernel}
Arthur Gretton, Karsten~M Borgwardt, Malte~J Rasch, Bernhard Sch{\"o}lkopf, and
  Alexander Smola.
\newblock A kernel two-sample test.
\newblock \emph{Journal of Machine Learning Research}, 13\penalty0
  (Mar):\penalty0 723--773, 2012.

\bibitem[Guo et~al.(2017)Guo, Pleiss, Sun, and Weinberger]{guo2017calibration}
Chuan Guo, Geoff Pleiss, Yu~Sun, and Kilian~Q Weinberger.
\newblock On calibration of modern neural networks.
\newblock \emph{arXiv preprint arXiv:1706.04599}, 2017.

\bibitem[He et~al.(2016)He, Zhang, Ren, and Sun]{he2016deep}
Kaiming He, Xiangyu Zhang, Shaoqing Ren, and Jian Sun.
\newblock Deep residual learning for image recognition.
\newblock In \emph{CVPR}, 2016.

\bibitem[Hendrycks \& Gimpel(2017)Hendrycks and Gimpel]{hendrycks2016baseline}
Dan Hendrycks and Kevin Gimpel.
\newblock A baseline for detecting misclassified and out-of-distribution
  examples in neural networks.
\newblock \emph{ICLR}, 2017.

\bibitem[Hinton et~al.(2015)Hinton, Vinyals, and Dean]{hinton2015distilling}
Geoffrey Hinton, Oriol Vinyals, and Jeff Dean.
\newblock Distilling the knowledge in a neural network.
\newblock \emph{arXiv preprint arXiv:1503.02531}, 2015.

\bibitem[Huang et~al.(2016)Huang, Liu, and Weinberger]{huang2016densely}
Gao Huang, Zhuang Liu, and Kilian~Q Weinberger.
\newblock Densely connected convolutional networks.
\newblock \emph{arXiv preprint arXiv:1608.06993}, 2016.

\bibitem[Ji et~al.(2013)Ji, Xu, Yang, and Yu]{ji20133d}
Shuiwang Ji, Wei Xu, Ming Yang, and Kai Yu.
\newblock 3d convolutional neural networks for human action recognition.
\newblock \emph{IEEE transactions on pattern analysis and machine
  intelligence}, 35\penalty0 (1):\penalty0 221--231, 2013.

\bibitem[Kingma \& Ba(2014)Kingma and Ba]{kingma2014adam}
Diederik Kingma and Jimmy Ba.
\newblock Adam: A method for stochastic optimization.
\newblock \emph{arXiv preprint arXiv:1412.6980}, 2014.

\bibitem[Krizhevsky \& Hinton(2009)Krizhevsky and Hinton]{cifar}
Alex Krizhevsky and Geoffrey Hinton.
\newblock Learning multiple layers of features from tiny images.
\newblock 2009.

\bibitem[Krizhevsky et~al.(2012)Krizhevsky, Sutskever, and
  Hinton]{krizhevsky2012imagenet}
Alex Krizhevsky, Ilya Sutskever, and Geoffrey~E Hinton.
\newblock Imagenet classification with deep convolutional neural networks.
\newblock In \emph{Advances in neural information processing systems}, pp.\
  1097--1105, 2012.

\bibitem[Manning et~al.(1999)Manning, Sch{\"u}tze,
  et~al.]{manning1999foundations}
Christopher~D Manning, Hinrich Sch{\"u}tze, et~al.
\newblock \emph{Foundations of statistical natural language processing}, volume
  999.
\newblock MIT Press, 1999.

\bibitem[Moosavi-Dezfooli et~al.(2017)Moosavi-Dezfooli, Fawzi, Fawzi, and
  Frossard]{moosavi2016universal}
Seyed-Mohsen Moosavi-Dezfooli, Alhussein Fawzi, Omar Fawzi, and Pascal
  Frossard.
\newblock Universal adversarial perturbations.
\newblock \emph{CVPR}, 2017.

\bibitem[Nguyen et~al.(2015)Nguyen, Yosinski, and Clune]{nguyen2015deep}
Anh Nguyen, Jason Yosinski, and Jeff Clune.
\newblock Deep neural networks are easily fooled: High confidence predictions
  for unrecognizable images.
\newblock 2015.

\bibitem[Pereyra et~al.(2017)Pereyra, Tucker, Chorowski, Kaiser, and
  Hinton]{pereyra2017regularizing}
Gabriel Pereyra, George Tucker, Jan Chorowski, {\L}ukasz Kaiser, and Geoffrey
  Hinton.
\newblock Regularizing neural networks by penalizing confident output
  distributions.
\newblock \emph{ICLR}, 2017.

\bibitem[Pimentel et~al.(2014)Pimentel, Clifton, Clifton, and
  Tarassenko]{pimentel2014review}
Marco~AF Pimentel, David~A Clifton, Lei Clifton, and Lionel Tarassenko.
\newblock A review of novelty detection.
\newblock \emph{Signal Processing}, 99:\penalty0 215--249, 2014.

\bibitem[Sabokrou et~al.(2016)Sabokrou, Fayyaz, Fathy,
  et~al.]{sabokrou2016fully}
Mohammad Sabokrou, Mohsen Fayyaz, Mahmood Fathy, et~al.
\newblock Fully convolutional neural network for fast anomaly detection in
  crowded scenes.
\newblock \emph{arXiv preprint arXiv:1609.00866}, 2016.

\bibitem[Saito \& Rehmsmeier(2015)Saito and Rehmsmeier]{saito2015precision}
Takaya Saito and Marc Rehmsmeier.
\newblock The precision-recall plot is more informative than the roc plot when
  evaluating binary classifiers on imbalanced datasets.
\newblock \emph{PloS one}, 10\penalty0 (3):\penalty0 e0118432, 2015.

\bibitem[Schlegl et~al.(2017)Schlegl, Seeb{\"o}ck, Waldstein, Schmidt-Erfurth,
  and Langs]{schlegl2017unsupervised}
Thomas Schlegl, Philipp Seeb{\"o}ck, Sebastian~M Waldstein, Ursula
  Schmidt-Erfurth, and Georg Langs.
\newblock Unsupervised anomaly detection with generative adversarial networks
  to guide marker discovery.
\newblock In \emph{International Conference on Information Processing in
  Medical Imaging}, pp.\  146--157. Springer, 2017.

\bibitem[Simonyan \& Zisserman(2015)Simonyan and Zisserman]{simonyan2014very}
Karen Simonyan and Andrew Zisserman.
\newblock Very deep convolutional networks for large-scale image recognition.
\newblock \emph{ICLR}, 2015.

\bibitem[Sriperumbudur et~al.(2010)Sriperumbudur, Gretton, Fukumizu,
  Sch{\"o}lkopf, and Lanckriet]{sriperumbudur2010hilbert}
Bharath~K Sriperumbudur, Arthur Gretton, Kenji Fukumizu, Bernhard
  Sch{\"o}lkopf, and Gert~RG Lanckriet.
\newblock Hilbert space embeddings and metrics on probability measures.
\newblock \emph{Journal of Machine Learning Research}, 11\penalty0
  (Apr):\penalty0 1517--1561, 2010.

\bibitem[Sutherland et~al.(2016)Sutherland, Tung, Strathmann, De, Ramdas,
  Smola, and Gretton]{sutherland2016generative}
Dougal~J Sutherland, Hsiao-Yu Tung, Heiko Strathmann, Soumyajit De, Aaditya
  Ramdas, Alex Smola, and Arthur Gretton.
\newblock Generative models and model criticism via optimized maximum mean
  discrepancy.
\newblock \emph{ICLR}, 2016.

\bibitem[Szegedy et~al.(2014)Szegedy, Zaremba, Sutskever, Bruna, Erhan,
  Goodfellow, and Fergus]{szegedy2013intriguing}
Christian Szegedy, Wojciech Zaremba, Ilya Sutskever, Joan Bruna, Dumitru Erhan,
  Ian Goodfellow, and Rob Fergus.
\newblock Intriguing properties of neural networks.
\newblock \emph{NIPS}, 2014.

\bibitem[Theis et~al.(2015)Theis, Oord, and Bethge]{theis2015note}
Lucas Theis, A{\"a}ron van~den Oord, and Matthias Bethge.
\newblock A note on the evaluation of generative models.
\newblock \emph{ICLR}, 2015.

\bibitem[Vincent \& Bengio(2003)Vincent and Bengio]{vincent2003manifold}
Pascal Vincent and Yoshua Bengio.
\newblock Manifold parzen windows.
\newblock In \emph{Advances in neural information processing systems}, pp.\
  849--856, 2003.

\bibitem[Wasserman(2006)]{nonparametric}
Larry Wasserman.
\newblock \emph{All of Nonparametric Statistics}.
\newblock Springer, 2006.

\bibitem[Xu et~al.(2015)Xu, Ehinger, Zhang, Finkelstein, Kulkarni, and
  Xiao]{xu2015turkergaze}
Pingmei Xu, Krista~A Ehinger, Yinda Zhang, Adam Finkelstein, Sanjeev~R
  Kulkarni, and Jianxiong Xiao.
\newblock Turkergaze: Crowdsourcing saliency with webcam based eye tracking.
\newblock \emph{arXiv preprint arXiv:1504.06755}, 2015.

\bibitem[Yu et~al.(2015)Yu, Zhang, Song, Seff, and Xiao]{yu15lsun}
Fisher Yu, Yinda Zhang, Shuran Song, Ari Seff, and Jianxiong Xiao.
\newblock Lsun: Construction of a large-scale image dataset using deep learning
  with humans in the loop.
\newblock \emph{arXiv preprint arXiv:1506.03365}, 2015.

\bibitem[Zagoruyko \& Komodakis(2016)Zagoruyko and Komodakis]{wide}
Sergey Zagoruyko and Nikos Komodakis.
\newblock Wide residual networks.
\newblock \emph{arXiv preprint arXiv:1605.07146}, 2016.

\bibitem[Zhang et~al.(2017)Zhang, Bengio, Hardt, Recht, and
  Vinyals]{zhang2016understanding}
Chiyuan Zhang, Samy Bengio, Moritz Hardt, Benjamin Recht, and Oriol Vinyals.
\newblock Understanding deep learning requires rethinking generalization.
\newblock \emph{ICLR}, 2017.

\end{thebibliography}
\bibliographystyle{iclr2018_conference}

\clearpage
\begin{appendix}
\vspace{-0.2cm}
\clearpage
%
%
%
%
%
%
\section{Supplementary Results in Section~\ref{sec::temperature} and \ref{sec::preprocessing}}\label{appendix::analysis}

\begin{figure}[h]
	\centering
	\includegraphics[width=\linewidth]{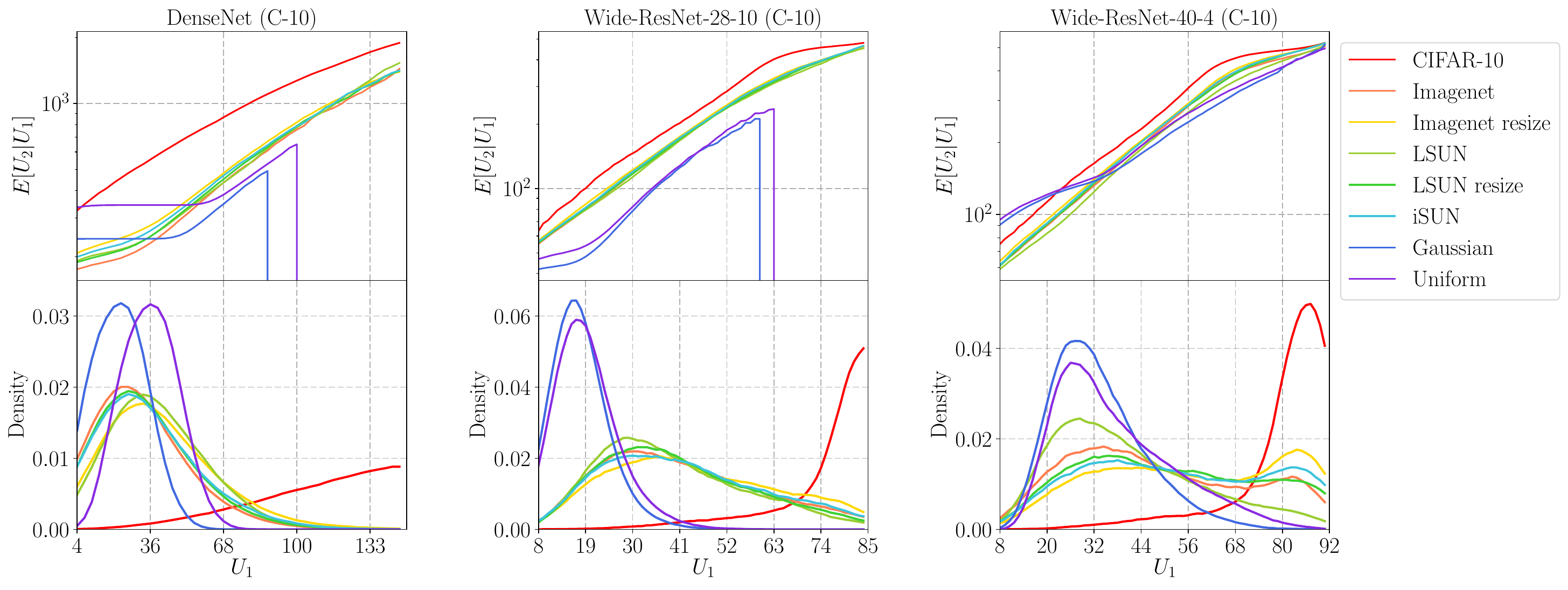}\vspace{-0.2cm}
	\caption{\small Expectation of the second order term $U_{2}$ conditioned on the first order term $U_{1}$ under DenseNet, Wide-ResNet-28-10 and Wide ResNet-40-4. All networks are trained on CIFAR-10.}
	\vspace{-0.3cm}
\end{figure}

\begin{figure}[h]
	\centering
	\includegraphics[width=\linewidth]{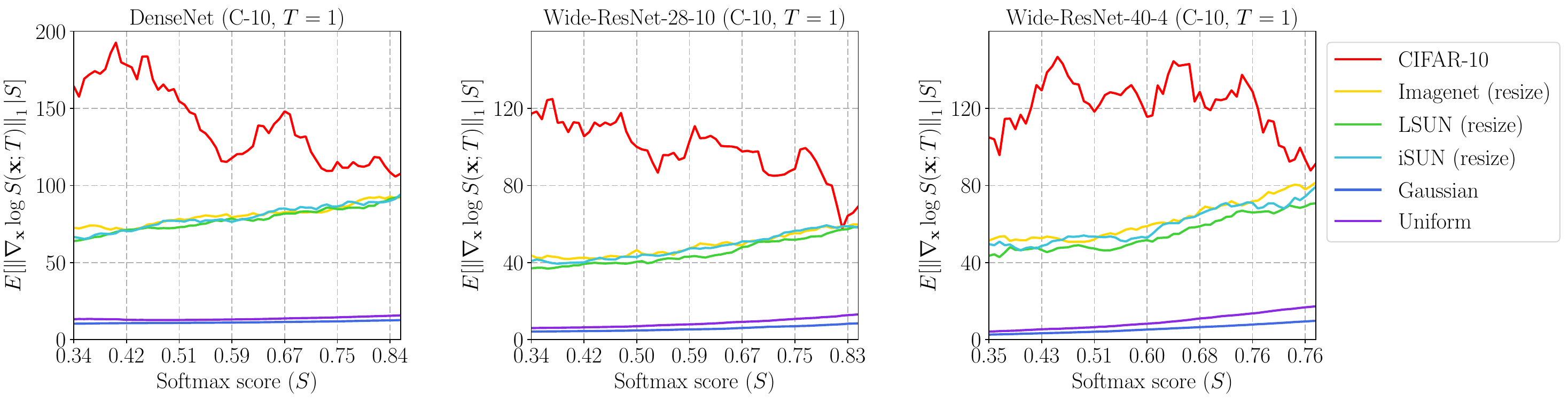}\vspace{-0.2cm}
	\caption{\small Expectation of gradient norms conditioned on the softmax scores under DenseNet, Wide-ResNet-28-10 and Wide ResNet-40-4,  where the temperature scaling is not used. All networks are trained on CIFAR-10.}
	\vspace{-0.3cm}
\end{figure}

\begin{figure}[h]
	\centering
	\includegraphics[width=\linewidth]{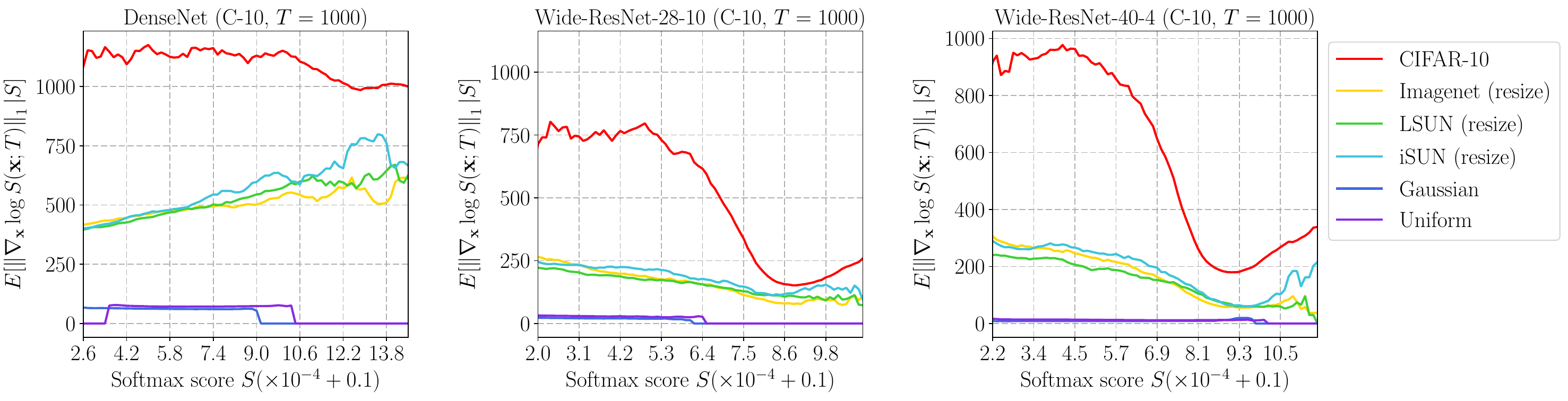}\vspace{-0.2cm}
	\caption{\small Expectation of gradient norms conditioned on the softmax scores under DenseNet, Wide-ResNet-28-10 and Wide ResNet-40-4, where the optimal temperature is used, i.e., $T=1000$. All networks are trained on CIFAR-10.}
	\vspace{-0.3cm}
\end{figure}

\clearpage

\section{Taylor Expansion}\label{appendix::taylor}
In this section, we present the Taylor expansion of the soft-max score function:
\begin{align*}
S_{\hat{y}}(\bm{x};T)&=\frac{\exp\left(f_{\hat{y}}(\bm{x})/T\right)}{\sum_{i=1}^{N}\exp(f_{i}(\bm{x})/T)}\\
&=\frac{1}{\sum_{i=1}^{N}\exp\left(\frac{f_{i}(\bm{x})-f_{\hat{y}}(\bm{x})}{T}\right)}\\
&=\frac{1}{\sum_{i=1}^{N}\left[1+\frac{f_{i}(\bm{x})-f_{\hat{y}}(\bm{x})}{T}+\frac{1}{2!}\frac{(f_{i}(\bm{x})-f_{\hat{y}}(\bm{x}))^{2}}{T^{2}}+o\left(\frac{1}{T^{2}}\right)\right]}&&\text{by Taylor expansion}\\
&\approx\frac{1}{N-\frac{1}{T}\sum_{i=1}^{N}[f_{\hat{y}}(\bm{x})-f_{i}(\bm{x})]+\frac{1}{2T^{2}}\sum_{i=1}^{N}[f_{i}(\bm{x})-f_{\hat{y}}(\bm{x})]^{2}}
\end{align*}

\section{Proposition~\ref{prop::1}}\label{appendix::prop1}
The following proposition~\ref{prop::1} shows that the detection error $P_{e}(T,0)\approx c$ if $T$ is sufficiently large. Thus, increasing the temperature further can only slightly improve the detection performance.
\begin{proposition}~\label{prop::1}
There exists a constant $c$ only depending on function $U_{1}$, in-distribution $P_{\bm{X}}$ and out-of-distribution $Q_{\bm{X}}$ such that 
$\lim_{T\rightarrow\infty}P_{e}(T,\varepsilon)=c$, when $\varepsilon=0$ (i.e., no input preprocessing). \vspace{-0.3cm}
\end{proposition}

\begin{proof}
Since 
$$S_{\hat{y}}(\bm{X};T)=\frac{\exp({f_{\hat{y}}(\bm{X})/T)}}{\sum_{i=1}^{N}\exp(f_{i}(\bm{X})/T)}=\frac{1}{1+\sum_{i\neq \hat{y}}\exp([f_{i}(\bm{X})-f_{\hat{y}}(\bm{X})]/T)}$$
Therefore, for any $\bm{X}$,
\begin{align*}
\lim_{T\rightarrow \infty}T\left(-\frac{1}{S_{\hat{y}}(\bm{X};T)}+N\right)&=\lim_{T\rightarrow \infty}\sum_{i\neq \hat{y}}T\left[1-\exp\left(\frac{f_{i}(\bm{X})-f_{\hat{y}}(\bm{X})}{T}\right)\right]\\
&=\sum_{i\neq \hat{y}}[f_{\hat{y}}(\bm{X})-f_{i}(\bm{X})]=(N-1)U_{1}(\bm{X})
\end{align*}
This indicates that the random variable 
$$T\left(-\frac{1}{S_{\hat{y}}(\bm{X};T)}+N\right)\rightarrow (N-1)U_{1}(\bm{X})\quad a.s.$$
as $T\rightarrow \infty$. This means that for a specific $\alpha>0$, choosing the threshold $\delta_{T}=1/(N-\alpha/T)$, then the false positive rate 
\begin{align*}
\text{FPR}(T)=Q_{\bm{X}}(S_{\hat{y}}(\bm{X};T)>1/(N-\alpha/T))&=Q_{\bm{X}}\left(T\left(N-\frac{1}{S_{\hat{y}}(\bm{X};T)}\right)>\alpha\right)\\
&\xrightarrow{T\rightarrow\infty}Q_{\bm{X}}\left((N-1)U_{1}(\bm{X})>\alpha\right),
\end{align*}
and the true positive rate 
\begin{align*}
\text{TPR}(T)=P_{\bm{X}}(S_{\hat{y}}(\bm{X};T)>1/(N-\alpha/T))&=P_{\bm{X}}\left(T\left(N-\frac{1}{S_{\hat{y}}(\bm{X};T)}\right)>\alpha\right)\\
&\xrightarrow{T\rightarrow\infty}P_{\bm{X}}\left((N-1)U_{1}(\bm{X})>\alpha\right).
\end{align*}
Choosing $\alpha^{*}$ such that $P_{\bm{X}}\left((N-1)U_{1}(\bm{X})>\alpha^{*}\right)=0.95$, then TPR$(T)\rightarrow 0.95$ as $T\rightarrow\infty$ and at the same time FPR$(T)\rightarrow Q_{\bm{X}}\left((N-1)U_{1}(\bm{X})>\alpha^{*}\right)$ as $T\rightarrow \infty$.
There exists a constant $c$ depending on $U_{1}, P_{\bm{X}}, Q_{\bm{X}}$ and $P_{Z}$, such that
$$\lim_{T\rightarrow \infty}P_{e}(T,0)=0.05P(Z=0)+P(Z=1)Q_{\bm{X}}\left((N-1)U_{1}(\bm{X})>\alpha^{*}\right)=c.$$
\end{proof}

\section{Analysis of Temperature}\label{appendix::variance}
For simplicity of the notations, let $\Delta_{i}=f_{\hat{y}}-f_{i}$ and thus $\Delta=\{\Delta_{i}\}_{i\neq \hat{y}}$.  Besides, let $\bar{\Delta}$ denote the mean of the set $\Delta$. Therefore, $$\bar{\Delta}=\frac{1}{N-1}\sum_{i\neq\hat{y}}\Delta_{i}=\frac{1}{N-1}\sum_{i\neq\hat{y}}[f_{\hat{y}}-f_{i}]=U_{1}.$$
Equivalently, 
$$U_{1}=\text{Mean}(\Delta).$$
Next, we will show 
$$U_{2}=\frac{1}{N-1}\sum_{i\neq\hat{y}}[f_{\hat{y}}-f_{i}]^{2}=\xoverbrace{\frac{1}{N-1}\sum_{i\neq\hat{y}}[\Delta_{i}-\bar{\Delta}]^{2}}^{\text{Variance$^{2}$($\Delta$)}}+\xoverbrace{\bar{\Delta}^{2}}^{\text{Mean}^{2}(\Delta)}.$$
Since 
\begin{align*}
U_{2}&=\frac{1}{N-1}\sum_{i\neq\hat{y}}\Delta_{i}^{2} &&\text{by} \Delta_{i}=f_{\hat{y}}-f_{i}\\
&= \frac{1}{N-1}\sum_{i\neq\hat{y}}(\Delta_{i}-\bar{\Delta}+{\bar{\Delta}})^{2}\\
&=\frac{1}{N-1}\sum_{i\neq\hat{y}}[(\Delta_{i}-\bar{\Delta})^{2}-2(\Delta_{i}-\bar{\Delta})\bar{\Delta}+\bar{\Delta}^{2}]\\
&=\xunderbrace{\frac{1}{N-1}\sum_{i\neq\hat{y}}[\Delta_{i}-\bar{\Delta}]^{2}}_{\text{Variance}^{2}(\Delta)}-\xunderbrace{\frac{2\bar{\Delta}}{N-1}\sum_{i\neq \hat{y}}(\Delta_{i}-\bar{\Delta})}_{=0}+\xunderbrace{\bar{\Delta}^{2}}_{\text{Mean}^{2}(\Delta)}
\end{align*}
then
$$U_{2}=\text{Variance}^{2}(\Delta)+\text{Mean}^{2}(\Delta)$$

\section{Additional Results on Distance Measurement}
Apart from the Maximum Mean Discrepancy, we also calculate the Energy distance between in- and out-of-distribution datasets. Let $P$ and $Q$ denote two different distributions. Then the energy distance between distributions $P$ and $Q$ is defined as
$$D_{\text{energy}}^2(P,Q)=2\mathbb{E}_{V\sim P, W\sim Q}\|X-Y\|-\mathbb{E}_{V,V'\sim P}\|X-X'\|-\mathbb{E}_{W,W'\sim Q}\|Y-Y'\|.$$
Therefore, the energy distance between two datasets $V=\{V_{1},...,V_{m}\}\stackrel{iid}{\sim} P$ and $W=\{W_{1},...,W_{m}\}\stackrel{iid}{\sim} Q$ is defined as 
$$\widehat{D_{\text{energy}}}^2(P,Q)=\frac{2}{m^{2}}\sum_{i=1}^{m}\sum_{j=1}^{m}\|V_{i}-W_{j}\|-\frac{1}{{{m\choose{2}}}}\sum_{i\neq j}\|V_{i}-V_{j}\|-\frac{1}{{{m\choose{2}}}}\sum_{i\neq j}\|W_{i}-W_{j}\|.$$
In the experiment, we use the 2-norm $\|\cdot\|_{2}$.

\begin{center}
\begin{tabular}{ lllll } 
\toprule
\textbf{In-distribution} & \textbf{Out-of-distribution} & \textbf{MMD} &\textbf{Energy}\\
\textbf{datasets} & \textbf{Datasets} & \textbf{Distance} &\textbf{Distance}\\
\midrule
\multirow{5}{0.15 \linewidth}{CIFAR-100 } 
& Tiny-ImageNet (crop) &  0.41 & 2.25\\ 
& LSUN (crop) &  0.43 & 2.31\\ 
& Tiny-ImageNet (resize) &  0.088 &0.54\\ 
& LSUN (resize) &  0.12 &0.63\\
\midrule
\bottomrule
\end{tabular}
\end{center}

\clearpage

\begin{figure}[t]
\centering
  \vspace{-0.6cm}
  \includegraphics[width=0.7\linewidth]{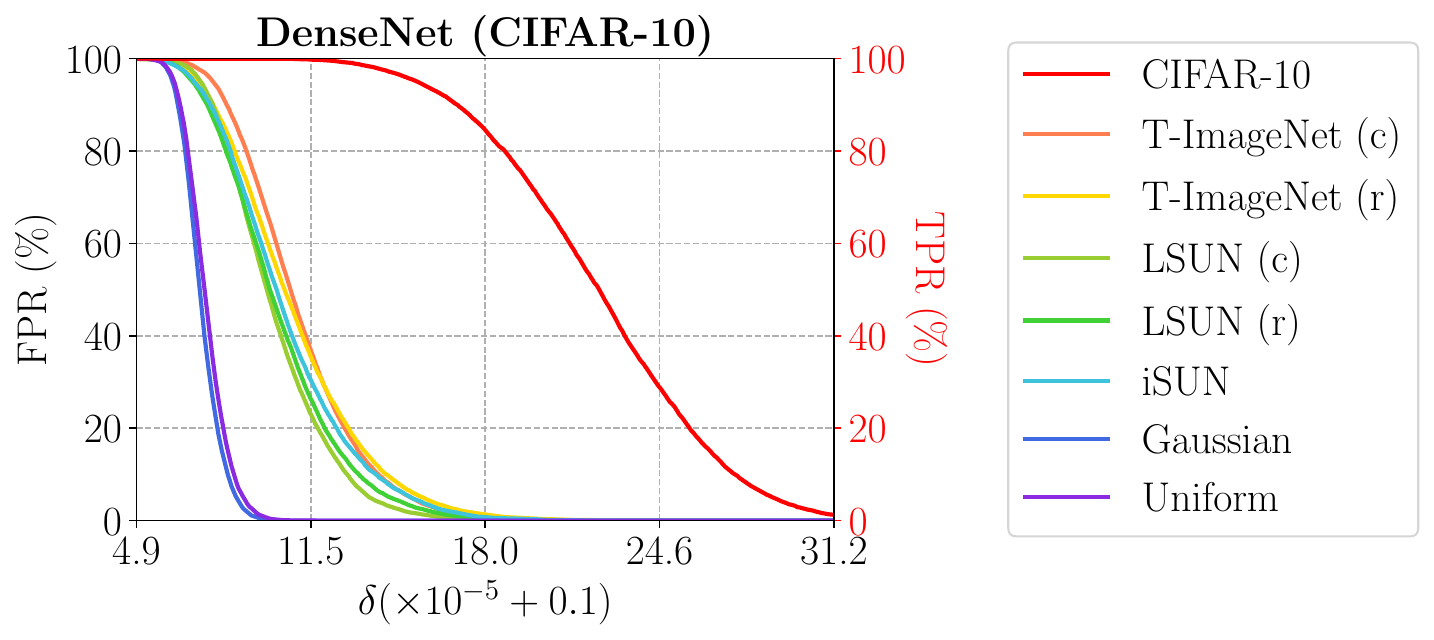}
  \caption{\small False positive rate (FPR) and true positive rate (TPR) under different thresholds ($\delta$) when the temperature $(T)$ is set to $1,000$ and  the perturbation magnitude ($\varepsilon$) is set to $0.0014$. The DenseNet is trained on CIFAR-10.}
  \label{fig::fpr-error}
  \vspace{-0.3cm}
  \end{figure}
  

\section{Additional Discussions}\label{appendix::robust}
In this section, we present additional discussion on the proposed method. We first empirically show how the threshold $\delta$ affects the detection performance. We next show how the proposed method performs when the parameters are tuned on a certain out-of-distribution dataset and are evaluated  on  other out-of-distribution datasets.  

\para{Effects of the threshold.} We analyze how the threshold  affects the following metrics: (1) \text{FPR}, i.e., the fraction of out-of-distribution images misclassified as in-distribution images; (2) \text{TPR}, i.e, the fraction of in-distribution images correctly classified as in-distribution images. In Figure~\ref{fig::fpr-error}, we show how the thresholds affect FPR and TPR when the temperature and perturbation magnitude are chosen optimally (i.e., $T=1,000$, $\varepsilon=0.0014$). From the figure, we can observe that the threshold corresponding to 95\% TPR can produce small FPRs on all out-of-distribution datasets.

\textbf{Difficult-to-classify images and difficult-to-detect images.} We analyze the correlation between the images that tend to be out-of-distribution and images on which the neural network tend to make incorrect predictions. To understand the correlation, we devise the following  experiment. For the fixed temperature $T$ and perturbation magnitude $\varepsilon$, we first set $\delta$ to the softmax score threshold corresponding to a certain true positive rate. Next, we calculate the test accuracy on the images with softmax scores above $\delta$ and  the test accuracy on the images with softmax score below $\delta$, respectively. We report the results  in  Figure~\ref{fig::accuracy}(a) and (b). From these two figures, we can observe that the images that are difficult to detect are more likely to be the images that are difficult to classify. For example, the DenseNet can achieve up to  98.5\% test accuracy on the images having softmax scores above the threshold corresponding to 80\% TPR, but can only achieve around 82\% test accuracy on the images having softmax scores below the threshold corresponding to 80\% TPR.

\begin{figure}[t]
	\centering
	\includegraphics[width=0.9\linewidth]{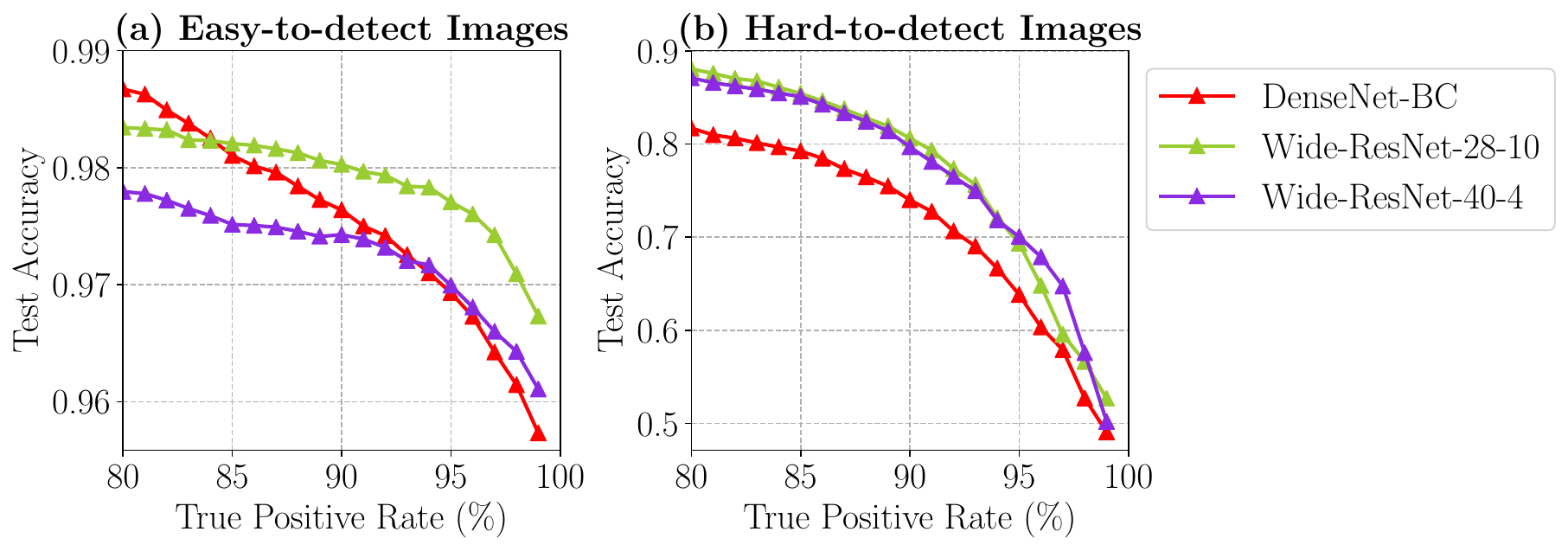}
	\caption{\small (a) The test accuracy on the images having softmax scores above the threshold corresponding to a certain true positive rate. (b) The test accuracy on the images having softmax scores below the threshold corresponding to a certain true positive rate. All networks are trained on CIFAR-10.}
	\label{fig::accuracy}
\end{figure}

\end{appendix}

\end{document}